\documentclass{article} 
\usepackage{nips13submit_e,times}
\usepackage{hyperref}
\usepackage{url}

\usepackage{amsmath, amssymb, amsthm}
\usepackage[pdftex]{graphicx}
\usepackage[multidot]{grffile}        
\usepackage{caption}
\usepackage{wrapfig}
\usepackage{algorithmic, algorithm}

\title{Nonparametric Weight Initialization of Neural Networks via Integral Representation}

\author{
Sho~Sonoda, Noboru~Murata \\
Schools of Advanced Science and Engineering \\
Waseda University \\
Shinjuku, Tokyo 169-8555 \\
\texttt{s.sonoda0110@toki.waseda.jp}, \\
\texttt{noboru.murata@eb.waseda.ac.jp} \\
}

\newcommand{\refalg}[1]{Alg.\ref{#1}}
\newcommand{\refeq}[1]{Eq.\ref{#1}}
\newcommand{\reffig}[1]{Fig.\ref{#1}}
\newcommand{\reftab}[1]{Tab.\ref{#1}}

\newcommand{\bx}{\mathbf{x}}
\newcommand{\ba}{\mathbf{a}}
\newcommand{\iid}{\overset{i.i.d.}{\sim}}
\newcommand{\supp}{\mathrm{supp\,}}

\newtheorem{thm}{Theorem}[section]

\newtheorem{prop}[thm]{Proposition}

\theoremstyle{definition}

\theoremstyle{remark}

\nipsfinalcopy 

\begin{document}

\maketitle

\begin{abstract}
A new initialization method for hidden parameters in a neural network is proposed.
Derived from the integral representation of neural networks,
a nonparametric probability distribution of hidden parameters is introduced.
In this proposal, hidden parameters are initialized by samples drawn from this distribution,
and output parameters are fitted by ordinary linear regression.
Numerical experiments show that backpropagation with proposed initialization converges faster
than uniformly random initialization.
Also it is shown that the proposed method achieves enough accuracy by itself without backpropagation in some cases.
\end{abstract}

\section{Introduction}
In the backpropagation learning of a neural network,
the initial weight parameters are crucial to its final estimates.
Since hidden parameters are put inside nonlinear activation functions,
simultaneous learning of all parameters by backpropagation is accompanied by a non-convex optimization problem.
When the machine starts from an initial point far from the goal,
the learning curve easily gets stuck in local minima or lost in plateaus,
and the machine fails to provide good performance.

Recently deep learning schemes draw tremendous attention for their overwhelming high performances for real world problems\cite{google.cat, dl.speech}.
Deep learning schemes consist of two stages: {\it pre-training} and {\it fine-tuning}.
The pre-training stage plays an important role for the convergence of the following fine-tuning stage.
In pre-training, the weight parameters are constructed layer by layer,
by stacking unsupervised learning machines such as restricted Boltzmann machines\cite{hinton} or denoising autoencoders\cite{bengio}.
Despite the brilliant progress in application fields,
theoretical interpretation of the schemes is still an open question\cite{bengio.replearn}.

In this paper we introduce a new initialization/pre-training scheme which could avoid the non-convex optimization problem.
The key concept is the probability distribution of weight parameters derived from Murata's integral representation of neural networks\cite{murata}.
The distribution gives an intuitive idea what the parameters represent and contains information about where efficient parameters exist.
Sampling from this distribution, we can initialize weight parameters more efficiently than just sampling from a uniform distribution.
In fact, for relatively simple or low dimensional problems, our method by itself attains a high accuracy solution without backpropagation.

De Freitas et al.\cite{defreitas} also introduced a series of stochastic learning methods for neural networks based on the Sequential Monte Carlo (SMC).
In their methods the learning process is iterative and initial parameters are given by less informative distributions such as normal distributions.
On the other hand we could draw the parameters from a {\it data dependent} distribution.
Furthermore, in SMC, the number of hidden units must be determined before the learning,
while it is determined naturally in our method.

\section{Back ground and related works}
One of the most naive initialization heuristics is to draw samples uniformly from an interval $[-\alpha, \alpha]$.
Nguyen and Widrow\cite{nguyen.widrow} gave two fundamental points of view.
First, since a typical activation function such as sigmoid and hyperbolic tangent is approximated as a linear function at its inflection point,
one should initialize the hidden parameters in such a way that the inputs for each hidden unit are in the {\it linear region}.
Second, since each hidden unit determines the slice of the {\it Fourier transformed} input space,
that is, each individual hidden unit responds {\it selectively} to only the inputs whose spatial frequency is in a particular band,
 one should initialize hidden parameters in such a way that the corresponding frequency bands cover the possible input frequencies.

LeCun et al.\cite{efficient.bp} also emphasized the need to preset parameters in the linear region
because parameters outside the linear region have small gradients and stray into more difficult nonlinear regions.
They focused on the curvature of input vectors and proposed to use $\alpha \propto m^{-1/2}$, where $m$ is the fan-in, or the dimensionality of input vectors.
Shimodaira\cite{shimodaira} proposed to initialize parameters such that
corresponding {\it activation} regions to cover whole the possible inputs.
Linear algebraic techniques are also employed.
For example, Shepanski\cite{shepanski} used the pseudo inverse to determine the parameters of linear approximated neural networks,
and Yam and Chow\cite{yam.chow} used the QR decomposition.

Integral transform viewpoints originated from more theoretical backgrounds than linear region viewpoints:
the theoretical evaluation of the approximation power of neural networks.
In the earliest stage, purely functional analysis methods were employed.
In 1957 Kolmogorov\cite{kolmogorov} showed that any multivariate continuous functions can be exactly
represented by sums of compositions of {\it different} continuous functions of only one variable.
Inspired by the Kolmogorov's theorem, Hecht-Nielsen\cite{hechtnielsen} and K\r{u}rkov\'{a}\cite{kurkova}
applied the idea to neural networks, which are sums of compositions of the {\it same} sigmoid function.
Sprecher\cite{sprecher} gave more constructive version of the proof and
later implemented the improved proof as a learning algorithm of neural networks\cite{sprecher1}.

In 1989 the universal approximation property of {\it single} layer neural networks
has been investigated 
and the integral transform aspects emerged.
Carroll and Dickinson\cite{carroll.dickinson} introduced the Radon transform 
and Funahashi\cite{funahashi} used the Fourier analysis and the Paley-Weiner theory,
whereas Cybenko\cite{cybenko} employed the Hahn-Banach and Riesz Representation theorems.
In the following years, upper bounds of the approximation error were investigated\cite{jones,barron,murata}.
Barron\cite{barron} refined the Jones' result\cite{jones} using the weighted Fourier transform.
K\r{u}rkov\'{a}\cite{kurkova} later developed the general theory of integral transforms.
Inspired by the Barron's result, Murata\cite{murata} introduced a family of integral transforms defined by
{\it ridge functions}, which are regarded as a hybrid of the Radon and wavelet transforms.
Cand\'{e}s\cite{candes} inherited Murata's transforms and developed ridgelets,
which was the beginning of the series of multiscale ``-lets'' analysis\cite{curvelets.ridgelets}.

Those multiscale viewpoints also inherits the {\it selective} activation properties of neural networks.
Denoeux and Lengell\'{e}\cite{denoeux} proposed to collect $K$ {\it prototype} vectors as initial hidden parameters.
Each prototype $p_k$ is drawn from its corresponding cluster $C_k$, where the clusters $\{C_k\}_{k=1}^K$ are formed in a stereographically projected input space.
In this manner each prototype $p_k$ comes to selectively respond to the input vectors $x$ which belongs to the cluster $C_k$.

This study is based on the integral transform viewpoint, and proposes a new way for practical implementation.
Although integral transforms have been well studied as theoretical integral representations of neural networks,
practical implementations for training have been merely done.
However integral representations have big advantage over linear region viewpoints
in that they can give global directions how each neural units should behave,
while the latter only give local directions.

\section{Nonparametric weight initialization via integral transform}
\subsection{Sampling based two-stage learning}
Let $g:\mathbb{R}^m \to \mathbb{R}$ be a neural network with a single hidden layer expressed as
\begin{equation}
g(\mathbf{x}) = \sum_{j=1}^J w_j \phi \left( \mathbf{a}_j \cdot \mathbf{x} - b_j \right) + w_0,
\end{equation}
where the map $\phi$ is called the {\it activation function};
$\ba_j$ and $b_j$ are called {\it hidden parameters}, and $w_j$ are {\it output parameters}.
With an ordinary {\it sigmoid function} $\sigma(z) := \frac{1}{1+\exp(-z)}$,
the activation function $\phi$ is supposed to be the {\it sigmoid pair} in the form
\begin{equation}\label{eq:sigpair}
\phi(z) := \frac{1}{H} \left\{ \sigma(z+h) - \sigma(z-h)\right\}, \quad (h>0),
\end{equation}
where $H:=\sigma(h)-\sigma(-h)$ normalizes the maximum value of $\phi$ to be one.

We consider an {\it oracle} distribution $p(\ba,b)$ of hidden parameters.
If such a distribution exists, we can sample and fix these hidden parameters according to $p(\ba,b)$ first,
and then we could fit the rest output parameters by ordinary linear regression.
We call this two-stage framework as {\it Sampling Regression (SR) learning}.

The candidates of $p(\ba,b)$ could be some parametric distributions such as normal distributions or uniform distributions.
In the following sections we derive a data dependent distribution from an integral representation of neural networks.

\subsection{Integral representations of neural networks}
Consider approximating a map $f:\mathbb{R}^m \to \mathbb{R}$ with a neural network.
Murata\cite{murata} defined an integral transform $T$ of $f$ with respect to a {\it decomposing kernel} $\phi_d$ as
\begin{equation}
T(\mathbf{a},b) := \frac{1}{C} \int_{\mathbb{R}^m} \phi_d( \ba \cdot \bx - b )f(\bx) d\bx,
\end{equation}
where $C$ is a normalizing constant.
Murata also showed that given the decomposing kernel $\phi_d$,
there exists the associating {\it composing kernel} $\phi_c$
such that for any $f \in L^1(\mathbb{R}^m) \cap L^p(\mathbb{R}^m) (1 \leq p \leq \infty)$, 
the inversion formula
\begin{equation}
f(\mathbf{x}) = \lim_{\epsilon \to 0} \int_{\mathbb{R}^{m+1}} \phi_c^*( \ba \cdot \bx - b ) T(\ba,b)e^{-\epsilon |\ba|^2}d\ba db \quad \mbox{in } L^p,
\end{equation}
holds (Th.1 in \cite{murata}) where $\cdot^*$ denotes the complex conjugate. The convergence factor $e^{-\epsilon |\ba|^2}$ is omitted when $T \in L^1(\mathbb{R}^{m+1})$,
which is attained when $f$ is compactly supported and $C^{m,\alpha}$-H\"{o}lder continuous with $0 < \alpha \leq 1$ (Th.3 in \cite{murata}),
or compactly supported and bounded $C^{m+1}$-smooth (Cor.2 in \cite{murata}).

In particular one can set a composing kernel $\phi_c$ as a sigmoid pair $\phi$ given in \refeq{eq:sigpair}
and the associating decomposing kernel as:
\begin{equation}\label{eq:drho}
\phi_d(z) = 
  \begin{cases}
    \rho^{(m)}(z) & \mbox{if $m$ is even} \\
    \rho^{(m+1)}(z) & \mbox{otherwise}
  \end{cases},
\end{equation}
where $\rho$ is a nonnegative $C^\infty$-smooth function whose support is in the interval $[-1,1]$.
Such a $\rho$ does exist and is known as a {\it mollifier}\cite{mollifier}.
The {\it standard mollifier} $\rho(z)=\exp \left( \frac{1}{z^2-1} \right)$ is a well-known example.

Hereafter we assume $\phi_c$ is a sigmoid pair and $\phi_d$ is the corresponding derivative of the standard mollifier.
We also assume that our target $f$ is a bounded and compactly supported $C^{(m+1)}$-smooth function.
Then the integral transform $T$ of $f$ is absolutely integrable
and the inversion formula is reduced to the direct form $f(\bx) = \int_{\mathbb{R}^{m+1}} \phi_c^*(\ba \cdot \bx - b) T(\ba,b)d\ba db$.

Let $\tau(\ba, b)$ be a probability distribution function over $\mathbb{R}^{m+1}$ which is proportional to $|T(\ba,b)|$,
and $c(\ba,b)$ be satisfying $c(\ba,b)\tau(\ba,b) = T(\ba,b)$ for all $(\ba,b) \in \mathbb{R}^{m+1}$.
With this notations, the inversion formula is rewritten as the expectation form with respect to $\tau(\ba,b)$, that is,
\begin{equation}
f(\bx) = \int_{\mathbb{R}^{m+1}} c(\ba,b) \phi_c(\ba \cdot \bx - b) \tau(\ba,b) d\ba db.
\end{equation}
The expression implies the finite sum
\begin{equation}
g_J(\bx) := \frac{1}{J} \sum_{j=1}^J c( \ba_j, b_j ) \phi_c( \ba_j \cdot \bx - b_j ), \quad (\ba_j, b_j) \iid \tau(\ba,b)
\end{equation}
converges to $f$ in mean square as $J \to \infty$, i.e. $\mathbb{E}[ g_J ] = f$ and $\mathrm{Var}[g_J] < \infty$ holds for any $J$ (Th.2 in \cite{murata}).
Here $g_J$ is a neural network with $2J$ hidden units,
therefore we can regard the inversion formula as an {\it integral representation} of neural networks.

\subsection{Practical calculation of the integral transform}
Now we attempt to make use of the integral transform $|T(\ba, b)|$ as an oracle distribution $p(\ba,b)$ of hidden parameters.
Although the distribution is given in the explicit form as we saw in the preceding section,
further refinements are required for practical calculation.

Given a set $\{ (\bx_n, y_n) \}_{n=1}^N \subset \mathbb{R}^m \times \mathbb{R}$ of input and output pairs,
$T(\ba,b)$ is empirically approximated as
\begin{equation} \label{eq:Tab.approx}
T(\ba,b) \approx \frac{1}{Z} \sum_{n=1}^N \phi_d( \ba \cdot \bx_n - b ) y_n,
\end{equation}
with some constant $Z>0$ which is hard to calculate exactly. 
In fact sampling algorithms such as the acceptance-rejection method\cite{prml} and Markov chain Monte Carlo method\cite{prml}
work with any unnormarized distribution because they only evaluate the ratio between probability values.
Note that the approximation converges to the exact $T(\ba,b)$ in probability by the law of large numbers
 {\it only} when the input vectors are i.i.d. samples from a uniform distribution.

As a decomposing kernel $\phi_d$ we make use of the $k$-th order derivative 
of the standard mollifier $\rho(z)=\exp{\frac{1}{z^2-1}}$ where $k=m$ if $m$ is even and $k=m+1$ otherwise.
The $k$-th derivative $\rho^{(k)}(z)$ of the mollifier takes the form
\begin{equation} \label{eq:deriv.mol}
\rho^{(k)}(z) = \frac{P_k(z)}{(z^2-1)^{2k}}\rho(z) \quad (k=0,1,2,\cdots),
\end{equation}
where $P_k(z)$ denotes a polynomial of $z$ which is calculated by the following recurrence formula:
\begin{eqnarray}
P_0(z) &\equiv& 1 \mbox{ (const.)}, \label{eq:P0} \\
P_{k+1}(z) &=&  P^\prime_k(z) (z^4-2z^2+1) + P_k(z) \left \{ -4kz^3 + 2(2k-1)z\right \} \label{eq:Pk}.
\end{eqnarray}
The higher order derivatives of a mollifier
has more rapid oscillations in the neighbourhoods of both edges of its support.

Given a data set $\mathcal{D}:=\{ (\bx_n,y_n)\}_{n=1}^N \subset \mathbb{R}^m \times \mathbb{R}$,
our {\bf Sampling Regression} method is summarized as below:
\begin{enumerate}
\item[0.] \underline{Preliminary stage}:
Calculate $\rho^{(k)}(z)$
according to \refeq{eq:deriv.mol}, \refeq{eq:P0} and \refeq{eq:Pk},
where $k=m$ if $m$ is even and $k=m+1$ otherwise.
Then $T(\ba,b)$ is calculated by \refeq{eq:Tab.approx} with setting $\phi_d=\rho^{(k)}$.
As we noted above, one can choose arbitrary $Z>0$.
\item \underline{Sampling stage}: Draw $J$ samples $\{(\ba_j,b_j)\}_{j=1}^J$ from 
the probability distribution $\tau(\ba,b) \propto |T(\ba,b)|$
by acceptance-rejection method,
where $J$ denotes the number of hidden (sigmoid pair) units.
Then we obtain the hidden parameters $\{ (\ba_j, b_j)\}_{j=1}^J$.
\item \underline{Regression stage}: Let $\phi_{jn} := \phi_d(\ba_j \cdot \bx_n - b_j)$ for all $j=1,\cdots,J$ and $n=1,\cdots,N$.
Solve the system of linear equations  $y_n = \sum_{j=1}^J w_j \phi_{jn} + w_0 \ (n=1\cdots N)$ with respect to $\{ w_j \}_{j=0}^J$.
Then we obtain the output parameters $\{ w_j \}_{j=0}^J$.
\end{enumerate}

\subsection{For more efficient sampling}
Generally $|T(\ba,b)|$ is ill-shaped and sampling from the distribution is difficult.
For example in \reffig{fig:Tab.of.sin} Left, samples drawn from $|T(\ba,b)|$ of $f(x)=\sin 2 \pi x$ with $x \in [-1,1]$ is plotted.
Whereas in \reffig{fig:Tab.of.sin} Right, the same distribution is transformed to another $(\boldsymbol{\alpha},\beta)$-coordinate system (which is explained below). The support of the distribution is reshaped into a rectangular, which implies sampling from $|T(\boldsymbol{\alpha},\beta)|$ is easier than doing from $|T(\ba,b)|$.
\begin{figure}[h]
  \begin{center}
    \begin{tabular}{c}
      \begin{minipage}{0.5\hsize}
        \begin{center}
          \includegraphics[width=6cm]{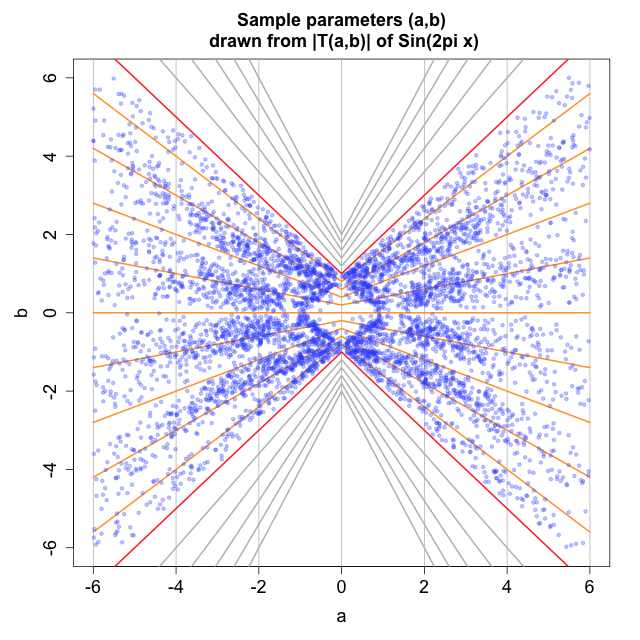}
        \end{center}
      \end{minipage}
      \begin{minipage}{0.5\hsize}
        \begin{center}
          \includegraphics[width=6cm]{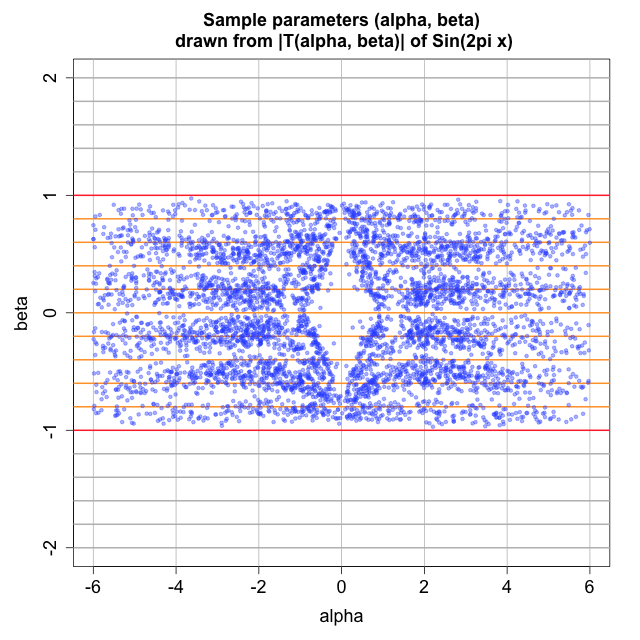}
        \end{center}
      \end{minipage}

    \end{tabular}
  \end{center}
  \caption{Sample parameters drawn from $|T(\ba,b)|$ of $\sin 2 \pi x$ case.
    Red lines indicate the theoretical boundary $b= \pm ( M \|\ba\| + 1 )$ of the support of $|T(\ba,b)|$.
    {\bf Left}: $|T(\ba,b)|$ has a non-convex support, in which case sampling is inefficient.
    {\bf Right}: The same sample points are plotted in the coordinate transformed $(\boldsymbol{\alpha},\beta)$-space.
    Coordinate lines are deformed to lattice lines, 
    and $|T(\boldsymbol{\alpha},\beta)|$ has a rectangular support.}
  \label{fig:Tab.of.sin}
\end{figure}

This ill-shapeness is formulated as following proposition.
\begin{prop} \label{prop:support}
Suppose the objective function $f(\bx)$ has a compact support,
then the support of its transform $T(\ba,b)$ is in the region 
$\Omega :=\{ (\ba,b) \,|\, |b| \leq M \| \ba \| + 1 \}$
 with $M := \max_{\bx \in \supp f} \|\bx\| $.
\end{prop}
\begin{proof}
Recall the support of $\phi_d$ is included in the interval $[-1,1]$,
therefore for any $\ba, b$ and $\bx$, $\phi_d( \ba \cdot \bx - b) \neq 0$ implies $| \ba \cdot \bx - b | < 1$.
%
The latter condition is equivalently deformed to $\ba \cdot \bx - 1 < b < \ba \cdot \bx +1$,
which implies $|b| < |\ba \cdot \bx|+1$.
By the compact support assumption of $f$,
taking the maximum with respect to $\bx$ leads to
$|b| < M \| \ba \|+1$.
By tracking back the inferences, 
for any $\ba,b$ and $\bx \in \supp f$, 
\begin{equation} \label{eq:prop1.4}
(\ba,b) \notin \Omega \Rightarrow \phi_d( \ba \cdot \bx - b ) = 0.
\end{equation}
Since for any $\bx \notin \supp f$, the integrand of $T(\ba,b)$ is always zero,
the integration domain of $T(\ba,b)$ can be restricted into $\supp f$.
Therefore by \refeq{eq:prop1.4}, 
\begin{equation}
T(\ba,b) \neq 0 \Rightarrow (\ba,b) \in \Omega
\end{equation}
holds, which comes to the conclusion:
$\supp T \subset \Omega$.
\end{proof}

In a relatively high dimensional input case,
sampling in the coordinate transformed $(\boldsymbol{\alpha},\beta)$-space
\begin{equation} \label{eq:ct}
\begin{pmatrix}
\ba \\ b
\end{pmatrix}
=
\begin{pmatrix}
\boldsymbol{\alpha} \\ (M \| \boldsymbol{\alpha} \| + 1)\beta
\end{pmatrix},
\end{equation}
is more efficient than sampling in the $(\ba,b)$-space
because the shape of the support of $|T(\ba,b)|$ in the $(\boldsymbol{\alpha},\beta)$-space
is rectangular (see, \reffig{fig:Tab.of.sin})
and therefore the proposal distribution is expected to reduce miss proposals, out of the support.

In case that the coordinate transform technique is not enough, 
it is worth sampling from each {\it component} distribution.
Namely, the empirically approximated $|T(\ba,b)|$ is bounded above by a {\it mixture distribution}:
\begin{eqnarray}\label{eq:mix}
|T(\ba,b)| \quad \approx \quad \frac{1}{Z} \Big| \sum_{n=1}^N y_n \phi_d( \ba \cdot \bx_n - b )  \Big| 
&\leq& \frac{1}{Z} \sum_{n=1}^N |y_n| |\phi_d( \ba \cdot \bx_n - b )|, \nonumber \\
&\propto& \sum_{n=1}^N \eta_n p_n( \ba, b ),
\end{eqnarray}
where $p_n( \ba,b ) \propto |\phi_d( \ba \cdot \bx_n - b )|$ is a {\it component distribution}
and $\eta_n \propto |y_n|$ is a {\it mixing probabilities}.

In addition, an upper bound of $\phi_d$ is given by the form
\begin{equation} \label{eq:log.bound}
\log | \phi_d(z) | \leq A z^2 + B,
\end{equation}
for some $A>0$ and $B$. 

\section{Experimental results}
We conducted three sets of experiments comparing three types of learning methods:
\begin{itemize}
\item[{\it BP}] Whole parameters are initialized by samples from a uniform distribution, and trained by \underline{B}ack\underline{P}ropagation.
\item[{\it SBP}] Hidden parameters are initialized by \underline{S}ampling from $|T(\ba,b)|$;
and the rest output parameters are initialized by samples from a uniform distribution.
Then whole parameters are trained by \underline{B}ack\underline{P}ropagation.
\item[{\it SR}] Hidden parameters are determined by \underline{S}ampling from $|T(\ba,b)|$; the rest output parameters are fitted by linear \underline{R}egression.
\end{itemize}

In order to compare the ability of the three methods, we conducted three experiments on three different problems:
One-dimensional complicated curve regression, Multidimensional Boolean functions approximation
and Real world data classification.

\subsection{One-dimensional complicated curve regression - Topologist's sine curve $\sin 2 \pi / x$}
\begin{figure}[h]
  \begin{center}
    \begin{tabular}{c}
      \begin{minipage}{0.5\hsize}
        \begin{center}
          \includegraphics[width=6cm]{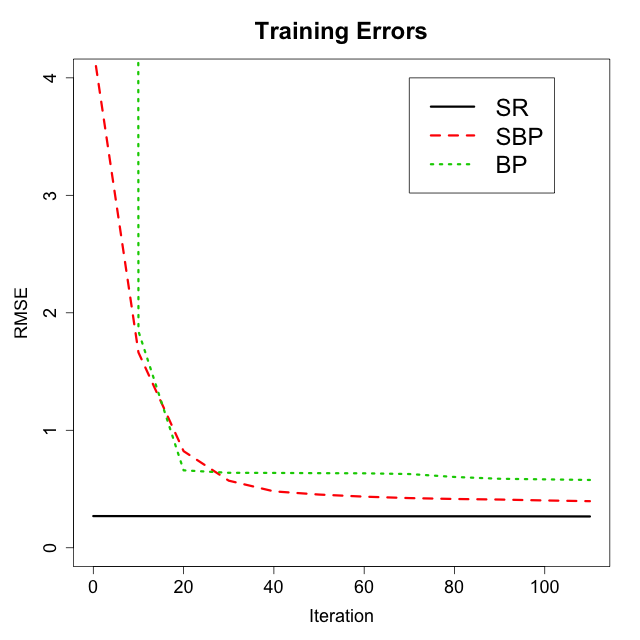}
        \end{center}
      \end{minipage}
      \begin{minipage}{0.5\hsize}
        \begin{center}
          \includegraphics[width=6cm]{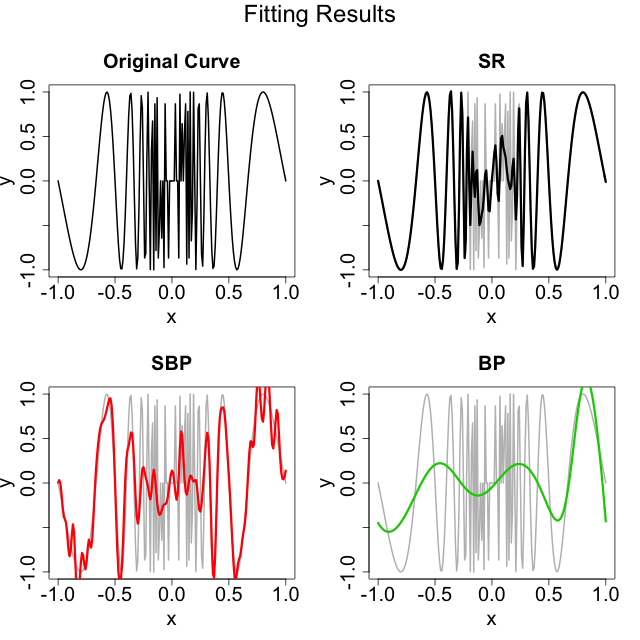}
        \end{center}
      \end{minipage}

    \end{tabular}
  \end{center}
  \caption{Training results of three methods for fitting the topologist's sine curve $\sin 2 \pi / x$.
    {\bf Left}: SR (solid black line) by itself achieved the highest accuracy without the iterative learning,
            whereas SBP (dashed red line) converged to lower RMSE than BP (dotted green line).
    {\bf Right}: The original curve (upper left) has high frequencies around the origin.
    SR (upper right) followed such a dynamic variation of frequency better than other two methods.
    SBP (lower left) roughly approximated the curve with noise.
    BP (lower right) only fitted moderate part of the curve.
    }
  \label{fig:TSC}
\end{figure}

First we performed one-dimensional curve regression.
The objective function is a two-sided {\it topologists's sine curve (TSC)} $f(x) := \sin 2 \pi / x$
defined on the interval $[-1,1]$ whose indefiniteness at zero is removed by defining $f(0) = 0$.
The TSC is such a complicated curve whose spatial frequency gets arbitrary high as $x$ tends to zero.
For training, $201$ points were sampled from the domain $[-1,1]$ in equidistant manner.
The number of hidden parameters were fixed to $100$ in each model.
Note that relatively redundant quantity of parameters are needed for
our sampling initialization scheme to obtain good parameters.
The output function was set to linear and the batch learning was performed by BFGS quasi-Newton method.
Uniformly random initialization parameters for BP and SBP were drawn from the interval $[-1,1]$.
Sampling from $|T(\ba,b)|$ was performed by acceptance-rejection method.

In \reffig{fig:TSC} Left, the Root Mean Squared Error (RMSE) in training phase of three methods are shown.
The solid black line corresponds to the result by SR, which by itself achieved the highest accuracy without iterative learnings.
The dashed red line corresponds to the result by SBP, and it converged to lower RMSE than that of BP depicted in the dotted green line.
In \reffig{fig:TSC} Right, fitting results of the three methods are shown.
As we noted the original curve (upper left) has numerical instability around the origin,
therefore it is difficult to fit the curve.
SR (upper right) approximated the original curve well except around the origin,
while other two methods, SBP (lower left) and BP (lower right) could just partly fit the original curve.
In this experiment, we examined the flexibility of our method by fitting a complicated curve.
The experimental result supports that the oracle distribution gave advantageous directions.

\subsection{Multidimensional Boolean functions approximation - Combined AND, OR and XOR}
\begin{wrapfigure}[24]{r}[1pt]{6cm}
  \vspace{-2\baselineskip}
  \begin{center}
          \includegraphics[width=6cm]{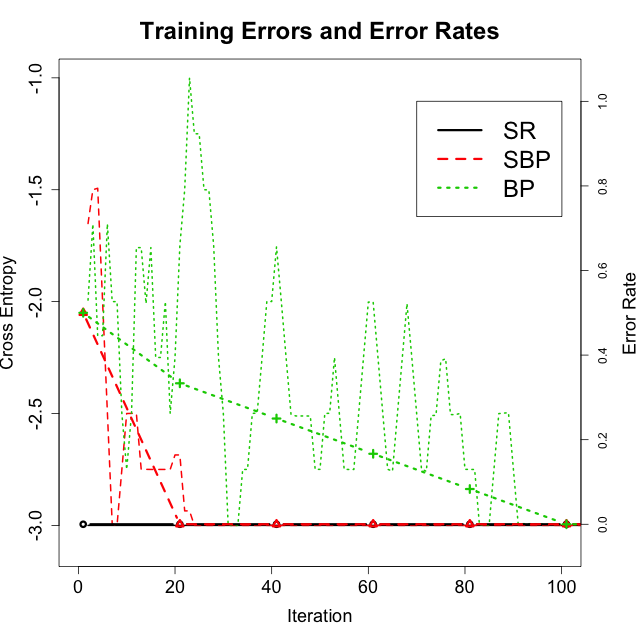}
	\captionof{figure}{Cross-entropy curves (thin lines) and classification error rates (thick lines).
      SR (solid black line) achieved the perfectly correct answer from the beginning.
      SBP (dashed red line) also attained the perfect solution faster than BP.
      BP (dotted green line) costed $100$ iterations of learning to give the correct answer.
       }
       \label{fig:logic}
  \end{center}
\end{wrapfigure}
Second we performed a binary problem with two-dimensional input and three-dimensional output.
Output vectors are composed of three logical functions: $F(x,y) := (x {\rm AND} y, x {\rm OR} y, x {\rm XOR} y)$.
Therefore the total number of data is just four: $(x,y) \in \{ (0,0), (0,1),(1,0),(1,1)\}$.
The number of hidden units were fixed to $10$.
The output function was set to sigmoid and the loss function was set to cross-entropy.
Uniformly random initialization parameters for BP and SBP were drawn from the interval $[-1,1]$.
Sampling from $|T(\ba,b)|$ was performed by acceptance-rejection method.

In \reffig{fig:logic} 
both the cross-entropy curves and classification error rates are depicted in thin and thick lines respectively.
The solid black line corresponds to the results by SR, which achieved the perfectly correct answer from the beginning.
The dashed red line corresponds to the results by SBP, which also attained the perfect solution faster than BP.
The dotted green line corresponds to the results by BP, which cost $100$ iterations of learning to give the correct answer.
In this experiment we have validated that the proposed method works well with multiclass classification problems.
The quick convergence of SBP indicates that $|T(\ba,b)|$ contains advantageous information
on the training examples to the uniform distribution.

\subsection{MNIST}
\begin{wrapfigure}[22]{r}[1pt]{6cm}
  \vspace{-4\baselineskip}
  \begin{center}
   \includegraphics[width=6cm]{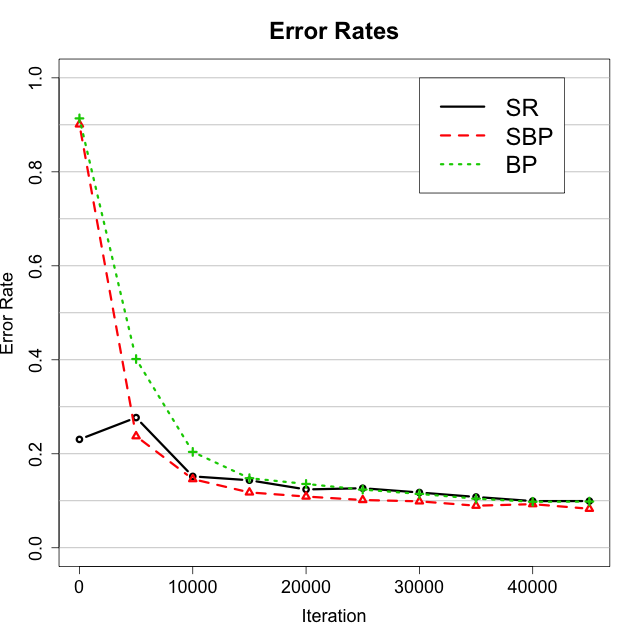}
  \end{center}
  \caption{Test classification error rates for MNIST dataset.
SR (black real line) marked $23.0\%$ at the beginning and finished $9.94\%$,
the error reascent suggests that SR may have overfitted.
SBP (red dashed line) reduced the fastest and finished $8.30\%$.
BP (green dotted line) declined the slowest and finished $8.77\%$.
}
  \label{fig:mnist}
\end{wrapfigure}

Finally we examined a real classification problem using the MNIST data set\cite{mnist}.
The data set consists of $60,000$ training examples and $10,000$ test examples.
Each input vector is a $256$-level gray-scaled $(28 \times 28 =)784$-pixel image of a handwritten digit.
The corresponding label is one of 10 digits.
We implemented these labels as $10$-dimensional binary vectors whose components are chosen randomly with equivalent probability for one and zero.
We used randomly sampled $15,000$ training examples for training and whole $10,000$ testing examples for testing.
The number of hidden units were fixed to $300$, which is the same size as used in the previous study of LeCun et al.\cite{lecun}.
Note that $J$ sigmoid pairs corresponds to $2J$ sigmoid units,
therefore we used $150$ sigmoid pairs for SR and SBP, and $300$ sigmoid units for BP.
The output function was set to sigmoid and the loss function was set to cross-entropy.
In obedience to LeCun et al.\cite{efficient.bp}, input vectors were normalized and 
randomly initialized parameters for BP and SBP were drawn from uniform distribution with mean zero and standard deviation $784^{-1/2} \approx 0.0357$.

Direct sampling from $|T(\ba,b)|$ is numerically difficult because
the differential order of its decomposing kernel $\phi_d$ piles up as high as $784$-th order.
We abandoned rigorous sampling and tried sampling from a {\it mixture annealed} distribution.
As described in \refeq{eq:mix}, we regarded $|T(\ba,b)|$ as a mixture of $|\phi_d(\ba \cdot \bx_n - b)|$.
By making use of the log boundary given by \refeq{eq:log.bound},
we numerically approximated $\phi_d(z)$ from above
\begin{equation}
\log |\phi_d(z)| \leq 2800 z^2 - 800, \quad (|z| < 1),
\end{equation}
and drew samples from an {\it easier} component distribution $p_n(\ba,b) \propto \exp\{2800 (\ba \cdot \bx_n - b)^2 - 800 \}$.
Details of the sampling technique is explained in \ref{supp:mix}.
The sampling procedure scales linearly with the dimensionality of the input space ($784$)
and the number of required hidden units ($150$) respectively.
In particular it scales constantly with the number of the training examples.

The following linear regression was conducted by singular value demcomposition (SVD),
which generally costs $O(m n^2)$ operations, assuming $m \geq n$, for decomposing a $m \times n$-matrix.
In our case $m$ corresponds to the number of the training examples ($15,000$)
and $n$ corresponds to the number of hidden units ($300$).
At last backpropagation learning was performed by stochastic gradient descent (SGD)
with adaptive learning rates and diagonal approximated Hessian\cite{bottou98}.
The experiment was performed in R\cite{statR} on a Xeon X5660 $2.8$GHz with $50$GB memory.

In \reffig{fig:mnist} the classification error rates for test examples are depicted.
The black real line corresponds to the results by SR,
which marked the lowest error rate ($23.0\%$) of the three at the beginning,
and finished $9.94\%$ after $45,000$ iterations of SGD training.
The training process was not monotonically decreasing in the early stage of training,
it appears that the SR initialization overfitted to some extent.
The red dashed line corresponds to the results by SBP,
which marked the steepest error reduction in the first $5,000$ iterations of SGD training
and finished $8.30\%$.
The green dotted line corresponds to the results by BP,
which declined the slowest in the early stage of training
and finished $8.77\%$.

In \reftab{tab:time} the training time from initialization through SGD training is listed.
The sampling step in SR ran faster than the following regression and SGD steps.
In addition, the sampling time of SR and SBP was as fast as the sampling time of BP.
As we expected, the regression step in SR, which scales linearly with the amount of the data, cost much more time than the sampling step did.
The SGD step also cost, however each step cost around merely $0.05$ seconds,
and it would be shorten if the initial parameters had better accuracy.

\begin{table}[htb] 
  \begin{center}
    \caption{Training Times for MNIST}
    \begin{tabular}{l||c|c|c}  
      \hline
      Method & Sampling [s] & Regression [s] & BP by SGD ($45,000$ itrs.) [s] \\ \hline
      SR & $1.15 \times 10^{-2}$ & $2.60$ & $2.00 \times 10^3$ \\
      SBP & $1.14 \times 10^{-2}$ & - & $2.31 \times 10^3$ \\
      BP & $1.15 \times 10^{-2}$ & - & $2.67 \times 10^3$ \\ \hline
    \end{tabular}
    \label{tab:time}
  \end{center} 
\end{table}

In this experiment, we confirmed that the proposed method still works for real world data
with the aid of an annealed sampling technique.
Although SR showed an overfitting aspects,
the fastest convergence of SBP supports that 
the oracle distribution gave meaningful parameters,
and the annealed sampling technique could draw meaningful samples.
Hence the overfitting of SR possibly comes from regression step,
which suggests the necessity for further blushing up of regression technique.
In addition, our further experiments also indicated
that when the number of hidden units increased to $6,000$,
the {\it initial} test error rate scored $3.66\%$,
which is smaller than the previously reported error rates
$4.7\%$ by LeCun et al.\cite{lecun} with $300$ hidden units.

\section{Conclusion and future directions}
In this paper, we introduced a two-stage weight initialization method for backpropagation:
sampling hidden parameters from the oracle distribution
and fitting output parameters by ordinary linear regression.
Based on the integral representation of neural networks,
we constructed our oracle distributions from given data in a nonparametric way.
Since the shapes of those distributions are not simple in high dimensional input cases,
we also discussed some numerical techniques such as the coordinate transform 
and the mixture approximation of the oracle distributions.
We performed three numerical experiments:
complicated curve regression, Boolean function approximation, and handwritten digit classification.
Those experiments show that our initialization method works well with backpropagation.
In particular for the low dimensional problems, 
well-sampled parameters by themselves achieve good accuracy without any parameter updates by backpropagation.
For the handwritten digit classification problem,
the proposed method works better than random initialization.

Sampling learning methods inevitably come with redundant hidden units
since drawing good samples usually requires a large quantity of trial.
Therefore the model shrinking algorithms such as pruning, sparse regression,
 dimension reduction and feature selection are naturally compatible to the proposed method.

Although plenty of integral transforms have been used for theoretical analysis of neural networks,
numerical implementations, in particular sampling approaches are merely done.
Even theoretical calculations often lack practical applicability,
for example a higher order of derivative in our case,
each integral representation interprets different aspects of neural networks.
Further Monte Carlo discretization of other integral representations is an important future work.

In the deep learning context, it is said that the deep structure remedies
the difficulty of a problem by multilayered superpositions of simple information transformations.
We conjecture that the complexity of high dimensional oracle distributions can be
decomposed into relatively simple distributions in each layer of the deep structure.
Therefore, extending our method to the multilayered structure is our important future work.

\subsubsection*{Acknowledgments}
The authors are grateful to Hideitsu Hino for his incisive comments on the paper.
They also thank to Mitsuhiro Seki for having constructive discussions with them.

\renewcommand{\refname}{\normalfont\normalsize\bfseries References}
\small{
\bibliographystyle{unsrt}
\bibliography{reference_shortshort}
}

\appendix
\section*{Supplementary materials}
\section{Sampling recipes}
Sampling hidden parameter $(\ba,b)$'s from the oracle distribution $p(\ba,b)$ demands a little ingenuity.
In our experiments, we have implemented two sampling procedures:
a rigorous but naive, computationally inefficient way
and an approximative/ad hoc but quick and well-performing way.
Although both work quickly and accurately in a low dimensional input problem,
only the latter works in a high dimensional problem such as MNIST.

\subsection{Sampling from rigorous oracle distribution}
Given a decomposing kernel $\phi_d(z) :=\rho^{(m)}(z)$,
we employed acceptance-rejection (AR)  method directly on rigorous sampling from $p(\ba,b)$
On a proposal distribution $q(\ba,b)$, we employed uniform distribution.
We assume here that 
the support $\Omega$ of proposal distribution $q(\ba,b)$ has been adjusted to
cover the {\it mass} of $p(\ba,b)$ as tight as possible,
and the infimum $k := \inf p(\ba,b)/q(\ba,b)$ has been estimated.
Then our sampling procedure is conducted according to the following \refalg{alg:naive}.

\begin{algorithm}
  \caption{Rigorous sampling according to ordinary acceptance-rejection method.}\label{alg:naive}
   \begin{algorithmic}
     \REPEAT
     \STATE draw proposal point $(\ba^*,b^*) \sim q(\ba,b)$.
     \STATE draw uniformly random value $u$ from the interval $[0,1]$.
     \IF{ $u \leq \frac{p(\ba^*,b^*)}{k q(\ba^*, b^*)}$ }
     \RETURN $(\ba^*,b^*)$ \COMMENT{accept}
     \ELSE
     \STATE do nothing \COMMENT{reject}
     \ENDIF
     \UNTIL{acceptance occurs.}
   \end{algorithmic}
\end{algorithm}

Note that in a high dimensional case, the estimation accuracy of $k$ and the tightness of $\Omega$ affects the sampling efficiency and accuracy materially.
In fact, the expectation number of trial to obtain one sample by AR is $k$ times,
which gets exponentially large as the dimensionality increases.
Since the support of the oracle distribution $p(\ba,b)$ is not rectangular,
sampling from coordinate transformed $p(\boldsymbol{\alpha},\beta)$ remedies the difficulty.
In addition, the high order differentiation in the decomposing kernel $\phi_d$ cause numerical unstability.

\subsection{Sampling from mixture annealed distribution}\label{supp:mix}
In order to overcome the high dimensional sampling difficulty,
we approximately regarded $p(\ba,b)$ as a mixture distribution $p(\ba,b) \approx \sum_{n=1}^N \eta_n p_n(\ba,b)$
(as described in \refeq{eq:mix}) and conducted two-step sampling:
first choose one component distribution $p_n(\ba,b)$ according to the mixing probability $\eta_n \propto |y_n|$,
second draw a sample $(\ba,b)$ from chosen component distribution $p_n(\ba,b)$.

Sampling from $p_n(\ba,b) \propto |\phi_d(\ba \cdot \bx_n - b)|$ holds another difficulty
due to its high order differentiation in $\phi_d(z)$.
According to its upper bound evaluation (\refeq{eq:log.bound}),
a high order derivative $\rho^{(m)}(z)(=\phi_d(z))$ has its almost all {\it mass} around both edge
of its domain interval $[-1,1]$ and almost no mass in the middle of the domain (see \reffig{fig:Dmol} Left).
Hence we approximated, or {\it annealed}, $\rho^{(m)}(z)$
by a beta distribution,
which could model extreme skewness of $\rho^{(m)}(z)$ (e.g., $\rm{Beta}(z; 100,3)$; see \reffig{fig:Dmol} Right).
Then we conducted further steps of sampling:
first sample $z \in [-1,1]$ according to the annealing beta distribution, 
then sample $\ba$ and $b$ under the restriction $z = \ba \cdot \bx_n - b$.

\begin{figure}[h]
  \begin{center}
    \begin{tabular}{c}
      \begin{minipage}{0.5\hsize}
        \begin{center}
          \includegraphics[width=6cm]{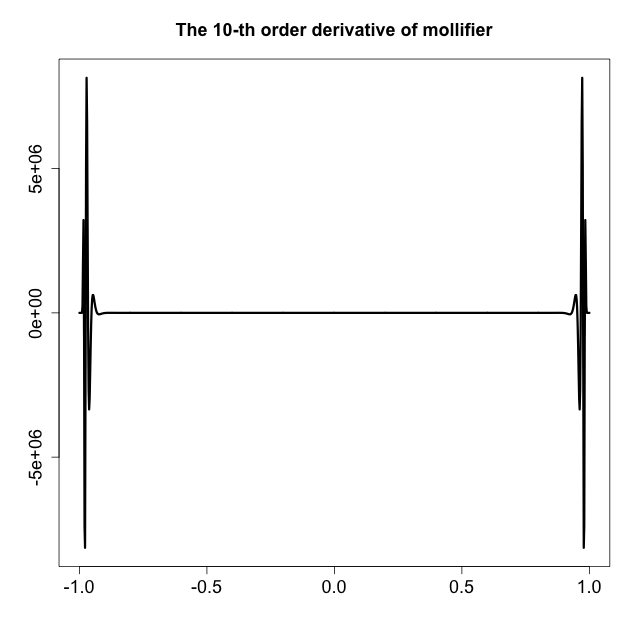}
        \end{center}
      \end{minipage}
      \begin{minipage}{0.5\hsize}
        \begin{center}
          \includegraphics[width=6cm]{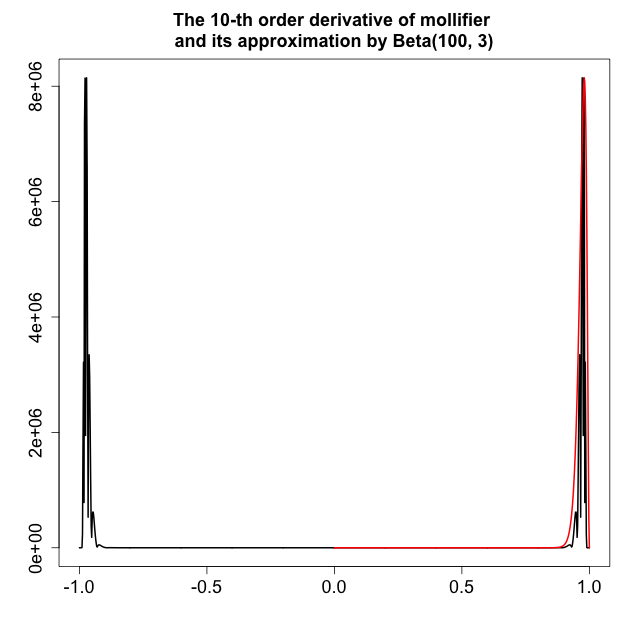}
        \end{center}
      \end{minipage}
    \end{tabular}
  \end{center}
  \caption{10-th order derivative $\rho^{(10)}(z)$ of mollifier.
    {\bf Left}: $\rho^{(10)}(z)$ has almost all mass, with high frequency, at both ends, and no mass in the middle of domain.
    {\bf Right}: The right half of $|\rho^{(10)}(z)|$ is approximated by beta distribution ${\rm Beta}(z; 100,3)$ (red line).
  \label{fig:Dmol}}
\end{figure}

Obviously the mixture approximation gives rise to poor restriction and virtual indefiniteness of $(\ba,b)$.
Since the rigorous computation establishes all relations between $(\ba,b)$ and all $\bx_n$'s,
whereas the mixture approximation does just one relation between $(\ba,b)$ and one particular $\bx_n$.
We introduced two additional assumptions.
First, $\ba$ is parallel to given $\bx_n$.
Since $\ba$ always appears in the form $\ba \cdot \bx_n$,
only the parallel component of $\ba$ could have any effect (on one particular $\bx_n$),
hence we eliminated the extra freedom in the orthogonal component.
Second, the norm $a := \| \ba \|$ has similar scale to the distances $\| \bx_n - \bx_m \|$ between input vectors.
Since $a$ controls the spatial frequency of a hidden unit,
it determines how broad the hidden unit covers the part of the input space.
Namely, $a$ controls which input vectors are selectively responded by the unit.
Therefore, in order to avoid such an isolation case that an unit responds for only one input vector,
we assumed $a$ is no smaller than the distance between input vectors.
In this procedure we set $a$ as a distance $\| \bx_n - \bx_m \|$ of randomly selected two input examples $\bx_n$ and $\bx_m$.
We denote this procedure simply by $a \sim p( \| \bx - \bx' \| )$.
Once $\ba$ is fixed with these assumptions, $b$ is determined as $b = \ba \cdot \bx_n - z$.

Given shape parameters $\alpha, \beta$ of the beta distribution ${\rm Beta}(z; \alpha, \beta)$,
one cycle of our second sampling method is summarized as \refalg{alg:mix}.
This method consists of no more expensive steps.
It scales linearly with the dimensionality of the input space and the number of required sample parameters respectively.
Moreover, it does not depends on the size of the training data.

\begin{algorithm}
  \caption{Quick sampling from mixture annealed distribution (for high dimensional use.)}\label{alg:mix}
   \begin{algorithmic}
     \STATE choose a suffix $n$ of $\bx_n$ according to the mixing probability $\eta_n$.
     \STATE draw $\zeta \sim \mathrm{Beta}(z; \alpha, \beta)$ and $k \sim \mathrm{Bernoulli}(k; p=0.5)$
     \STATE $z \leftarrow (-1)^k \zeta$
     \STATE set length $a \sim p( \bx - \bx' )$.
     \STATE $\ba \leftarrow a \bx_n / \| \bx_n \|$.
     \STATE $b \leftarrow \ba \cdot \bx_n - z$.
   \end{algorithmic}
\end{algorithm}

\end{document}